\documentclass{article} 
\usepackage{nips14submit_e,times}
\usepackage{hyperref}
\usepackage{url}

 \usepackage{relsize}
\usepackage{wrapfig}
\usepackage{balance}
\usepackage{array}
\usepackage{amsmath}
\usepackage{mathrsfs}
\usepackage{amssymb}
\usepackage{amsthm}
\usepackage{dsfont}
\newtheorem{lemma}{Lemma}
\newtheorem{theorem}{Theorem}
\newtheorem{definition}{Definition}

\newtheorem{remark}{Remark}

\newcommand*{\MyDef}{\mathrm{def}}
\newcommand*{\eqdefU}{\ensuremath{\mathop{\overset{\MyDef}{=}}}}
\newcommand*{\eqdef}{\mathop{\overset{\MyDef}{\resizebox{\widthof{\eqdefU}}{
\heightof{=}}{=}}}}

\newcommand{\cO}{\mathcal{O}}

\DeclareMathOperator*{\argmax}{arg\,max}

\newcommand{\beq}{\begin{equation}}
\newcommand{\eeq}{\end{equation}}

\newcommand{\beqa}{\begin{eqnarray}}
\newcommand{\eeqa}{\end{eqnarray}}

\newcommand{\beqan}{\begin{eqnarray*}}
\newcommand{\eeqan}{\end{eqnarray*}}

\newcommand{\beqal}{\begin{align*}}
\newcommand{\eeqal}{\end{align*}}

\newcommand{\ExtremeHunter}{{\sc ExtremeHunter}}
\newcommand{\ThresholdAscent}{{\sc ThresholdAscent}}
\newcommand{\UCB}{{\sc Ucb}}
\newcommand{\maxk}{{max-$k$}}

\usepackage{graphicx} 
\usepackage{subfigure}

\usepackage{algorithm}
\usepackage{algorithmic}

\title{Extreme bandits}

\author{
Alexandra Carpentier\\
Statistical Laboratory, CMS\\ 
University of Cambridge, UK\\
\texttt{\small a.carpentier@statslab.cam.ac.uk} \\
\And
Michal Valko \\
SequeL team\\
INRIA Lille - Nord Europe, France\\
\texttt{\small michal.valko@inria.fr} \\
}

\nipsfinalcopy 

\begin{document}

 \maketitle


\begin{abstract}
In many areas of medicine, security,  and life sciences, we
want to allocate limited resources to different sources
in order to detect extreme values.
In this paper, we study an efficient
way to allocate these resources \textit{sequentially}
under \textit{limited feedback}.
While sequential design of experiments
is well studied in \textit{bandit theory},
the most commonly optimized property is the regret
with respect to the maximum mean reward.
However, in other problems such as network intrusion
detection, we are interested in detecting the most extreme value output by the
sources.
Therefore, in our work we study 
\textit{extreme regret} which
measures the efficiency of an algorithm
compared to the oracle policy selecting
the source with the \textit{heaviest tail}.
We propose the \ExtremeHunter{} algorithm,
provide its analysis, and evaluate it empirically 
on synthetic and real-world experiments.
\end{abstract}


\section{Introduction}
\label{sec:intro}

We consider problems where the goal is to detect \textit{outstanding
events} or \textit{extreme values} in domains such as \textit{outlier detection}
\cite{abe2006outlier},
 \textit{security} \cite{priebe2005scan}, or \textit{medicine}
\cite{neill2010multivariate}. The detection of extreme values is important
in many life sciences, such as epidemiology, astronomy, or hydrology, 
where, for example, we may want to know the peak water flow. 
We are also motivated by \textit{network intrusion
detection} where the objective is to find the network node that was
compromised, e.g.,~by seeking the one creating the most number of outgoing
connections at once. 
The search for extreme events is typically studied in the field
of \textit{anomaly detection}, where one seeks to find examples
that are far away from the
majority, according to some problem-specific distance
(cf.~the surveys~\cite{chandola2009anomaly,markou2003novelty}).


In anomaly detection research, the concept of anomaly is
ambiguous and several definitions exist~\cite{markou2003novelty}:
point anomalies, structural anomalies, contextual anomalies, etc. These
definitions are
 often followed by heuristic approaches
that are seldom analyzed theoretically.
Nonetheless, there exist some theoretical characterizations of anomaly
detection. For instance, Steinwart et al.\,\cite{steinwart2005classification} 
consider the level
sets of the distribution underlying the data, and rare events corresponding to
\textit{rare level sets} are then identified as anomalies.
A very challenging characteristic of many problems in anomaly detection is
that the data emitted by the sources tend to be 
\textit{heavy-tailed} (e.g.,~network
traffic~\cite{agosta2013mixture}) and  anomalies come from the sources with
the heaviest distribution tails. In this case, rare level sets
of~\cite{steinwart2005classification} correspond to distributions' tails
and anomalies to extreme values. Therefore, we focus on the kind of anomalies that are
characterized by their \textit{outburst of events} or \textit{extreme values},
as in the setting of~\cite{turner2010fast} and~\cite{neill2010multivariate}.

Since in many cases, the 
collection of the data samples emitted by the sources is costly, it is
important to design \textit{adaptive-learning} strategies that
spend more time sampling sources that have a higher risk of being abnormal. 
The main objective of our work is the \textit{active allocation} of the sampling
resources for anomaly detection, in the setting where anomalies are defined as 
 extreme values.
Specifically, we consider a variation of the common setting of
\textit{minimal feedback} also known as the \textit{bandit
setting}~\cite{lai1985asymptotically}: the \textit{learner} searches for the
most
extreme value that the sources output by
probing the sources \emph{sequentially}.
In this setting, it must carefully decide which sources
to observe because  it only receives the observation from the source it chooses
to
observe. 
As a consequence, it needs to allocate the \textit{sampling time} efficiently
and should not waste it on sources that do not have an abnormal character. We
call this specific setting \textit{extreme bandits}, but it is also known as  
\maxk{} problem~\cite{cicirello2005max,streeter2006asymptotically,streeter2006simple}.
We emphasize that extreme bandits are poles apart from
classical bandits, where the objective is to maximize the sum of
observations~\cite{auer2002finite}. An effective algorithm for the classical
bandit setting should focus on the source with the highest mean, while an
effective algorithm for the extreme bandit problem should focus on 
the source with the heaviest tail. It is often the case that a heavy-tailed source has a small mean, which implies that the
classical bandit algorithms perform poorly for the extreme bandit problem.

The challenging part of our work dwells in the active sampling strategy  to
detect the heaviest tail under the limited bandit feedback.
We proffer \ExtremeHunter, a theoretically founded
algorithm, that sequentially allocates the resources in an
efficient way, for which we prove \textit{performance guarantees}. Our 
algorithm is efficient under a mild semi-parametric 
assumption common in \textit{extreme value theory}, while known
results by~\cite{cicirello2005max,streeter2006asymptotically,streeter2006simple}
for the extreme bandit problem only hold in a parametric setting
(see Section~\ref{sec:related} for a detailed comparison).


\section{Learning model for extreme bandits}
\label{sec:bandit}

In this section, we formalize the \textit{active (bandit) setting} and
characterize the measure of performance for any algorithm $\pi$.
The \textit{learning setting} is defined as follows. Every time step, each of
the $K$ \textit{arms} (sources) emits a sample $X_{k,t} \sim P_k$,
unknown to the learner. The precise characteristics of $P_k$ are defined in
Section~\ref{sec:setting}.
The learner $\pi$ then chooses some arm $I_t$ and then receives only the
sample $X_{I_t,t}$.
The performance of $\pi$  is evaluated by the most extreme value found
and compared to the most extreme value possible. We define the reward
of a learner $\pi$  as:
 \begin{align*}
 G^\pi_n =  \max_{t\leq n} X_{I_t,t}
 \end{align*}
 The optimal oracle strategy is the one that chooses at each time the arm with
the highest
potential revealing the highest value, i.e., the arm $*$ with the heaviest tail.
Its expected reward is then:
 \begin{align*}
 \mathbb  E  \left[ G_n^* \right] =  \max_{k\leq K}\mathbb E \left[\max_{t\leq 
n}X_{k,t}\right]
 \end{align*}
The goal of learner $\pi$ is to get as close as possible to the optimal oracle
strategy. In other words, the aim of $\pi$ is to minimize the expected 
\textit{extreme regret}:
\renewcommand{\baselinestretch}{1.2} 
\begin{definition}\label{def:extremeregret}
The extreme regret in the bandit setting is defined as:
\[
\mathbb E \left[ R^\pi_n \right] =  \mathbb E \left[ G_n^* \right] - \mathbb E 
\left[G^\pi_n\right] = 
\max_{k\leq
K} \mathbb E \left[ \max_{t\leq n}    X_{k,t} \right]
- \mathbb E \left[ \max_{t \leq n} X_{I_t,t} \right] \]
\end{definition}
 \renewcommand{\baselinestretch}{1} 
\section{Heavy-tailed distributions}
\label{sec:setting}

In this section, we formally define our observation model. Let $X_1,\ldots, X_n$
be $n$ i.i.d.\,observations from a distribution $P$. The
behavior of the statistic $\max_{i \leq n} X_i$ is studied by
\textit{extreme value theory}. One of the main results is the
Fisher-Tippett-Gnedenko theorem
\cite{fisher1928limiting,gnedenko1943distribution}
that characterizes the limiting distribution of
this maximum as $n$ converges to infinity.  Specifically, it proves that a
rescaled version of this
maximum converges to one of the three
possible distributions: \textit{Gumbel}, \textit{Fr\'echet},  or
\textit{Weibull}.
This rescaling factor depends on $n$. To be concise, we
write ``$\max_{i \leq n} X_i$ converges to a distribution'' to refer to the
convergence of the
rescaled version to a given distribution. 
%
The Gumbel distribution
corresponds to the limiting distribution of the maximum of \textit{`not too
heavy tailed'} distributions, such as sub-Gaussian or sub-exponential
distributions. The Weibull distribution coincides with the behaviour
of the maximum of some specific \textit{bounded} random variables. Finally,
the Fr\'echet distribution corresponds to the limiting distribution of the
maximum of \textit{heavy-tailed} random variables. As many interesting problems
 concern heavy-tailed distributions, we focus on Fr\'echet
distributions in this work. The distribution function of a Fr\'echet random variable is defined for
$x\geq m$, and for two parameters $\alpha, s$ as:
\begin{align*}
 P(x) = \exp\left\{-\left(\frac{x-m}{s}\right)^{\alpha}\right\}.
\end{align*}

In this work, we consider positive distributions $P:[0, \infty) \rightarrow
[0,1]$. For $\alpha >0$, the
Fisher-Tippett-Gnedenko theorem also states that
the statement `$P$ converges to an $\alpha$-Fr\'echet
distribution' is equivalent to the statement `$1-P$ is a $-\alpha$ regularly
varying function in the tail'. These statements are slightly less 
restrictive than the  definition of \textit{approximately} $\alpha$-Pareto 
distributions\footnote{We
recall the
definition of the standard Pareto distribution as a distribution $P$, where for
some constants $\alpha$ and  $C$, we have that for $x\geq C^{1/\alpha}$, $P =
1 -
Cx^{-\alpha}$.}, i.e.,~that there exists $C$ such that $P$ verifies:
\begin{align}\label{eq:asspyareto}
\lim_{x \rightarrow \infty} \frac{\left|1 - P(x) -
Cx^{-\alpha}\right|}{x^{-\alpha}} =0,
\end{align}
or equivalently that $P(x) = 1 - Cx^{-\alpha} + o(x^{-\alpha})$. If and only
if $1-P$ is $-\alpha$ regularly varying in the tail, then
the limiting
distribution of $\max_i X_i$
is an $\alpha$-Fr\'echet distribution. 
The assumption of $-\alpha$ regularly varying in the tail is thus the weakest 
possible assumption that ensures that the (properly rescaled) maximum of samples 
emitted by a heavy tailed distributions has a limit. Therefore, the very 
related assumption of approximate Pareto is almost minimal, but it is 
(provably) still not restrictive enough to ensure a convergence rate.  
For this reason, it is natural to introduce an assumption that is slightly 
stronger
than~\eqref{eq:asspyareto}.
 In particular, we assume, as it is common in the extreme
value literature, a \textit{second order} Pareto condition also known as the
\textit{Hall
condition} \cite{hall1984best}.
\begin{definition}
\label{def:sop}
A distribution $P$ is
$(\alpha,\beta,C,C')$-second order Pareto
($\alpha,\beta,C,C'>0$) if for $x\geq 0$:
\[\left|1 - P(x) - Cx^{-\alpha}\right| \leq C' x^{-\alpha(1+\beta)}.\]
\end{definition}
\renewcommand{\baselinestretch}{1} 
By this definition, 
$P(x) = 1 - Cx^{-\alpha} +
\cO\left(x^{-\alpha(1+\beta)}\right)$, which is stronger than the
assumption $P(x) = 1 -
Cx^{-\alpha} + o(x^{-\alpha})$, but similar for small $\beta$.

\renewcommand{\baselinestretch}{1.2} 
\begin{remark}
In the definition above, $\beta$ defines the rate of the convergence
(when $x$ diverges to infinity) of
the tail of  $P$ to the tail of a Pareto distribution $1 -
Cx^{-\alpha}$. The parameter $\alpha$ characterizes the heaviness of the tail:
The smaller the $\alpha$, the heavier the tail. In the remainder of the paper,
we will be therefore concerned with learning the $\alpha$ and identifying the
smallest one among the sources.
\end{remark}
\renewcommand{\baselinestretch}{1}

\section{Related work}
\label{sec:related}

There is a vast body of research in \textit{offline anomaly detection} which
looks for examples that deviate from the rest of the data, or that are
not expected from some underlying model. A comprehensive review of many anomaly
detection approaches can be found in \cite{markou2003novelty} or
\cite{chandola2009anomaly}. There has been also some work in active learning
for anomaly detection~\cite{abe2006outlier}, which uses a reduction
to classification. In \textit{online anomaly detection}, most of the research
focuses on studying the setting where 
a set of variables is monitored. A typical example is the monitoring of cold relief medications,
where we are interested in detecting an outbreak~\cite{neill2010multivariate}.
Similarly to our focus, these approaches do not look for outliers in a broad
sense but rather for the unusual burst of events~\cite{turner2010fast}.

In the extreme values settings above, it is often assumed,
that we have \textit{full information} about each variable.
This is in contrast to the \textit{limited feedback} or a \textit{bandit
setting} that we study in our work. 
There has been recently some interest in bandit algorithms for heavy-tailed
distributions~\cite{bubeck2013bandits}. However the goal
of~\cite{bubeck2013bandits} is radically different from ours as
they maximize the sum of rewards and not the maximal reward. Bandit algorithms
have been already used for network intrusion detection~\cite{liu2011dynamic},
but
they typically consider classical or restless setting. 
\cite{cicirello2005max,streeter2006asymptotically,streeter2006simple} were
the first to consider the extreme bandits problem, where our setting is defined
as the \maxk{} problem.
\cite{streeter2006asymptotically} and  \cite{cicirello2005max} consider
a fully parametric setting. The reward distributions are assumed
to be \textit{exactly generalized extreme value distributions}.
Specifically, \cite{streeter2006asymptotically} assumes that 
 the distributions are  exactly Gumbel, $P(x) = \exp(-(x-m)/s))$,
 and  \cite{cicirello2005max}, that the distributions are exactly of Gumbel or
Fr\'echet  $P(x) =
 \exp(-(x-m)^{\alpha}/(s\alpha)))$. Provided that these assumptions hold,
they propose an algorithm for which the regret is asymptotically negligible when
compared to the optimal oracle reward. These results are interesting since
they are the first for extreme bandits, but their parametric assumption
is unlikely to hold in practice and the asymptotic nature of their bounds
limits their impact. 
Interestingly, the objective of~\cite{streeter2006simple}
is to remove the parametric assumptions
of \cite{streeter2006asymptotically,cicirello2005max} by offering
the \ThresholdAscent{} algorithm. However, no analysis
of this algorithm for extreme bandits is provided. Nonetheless, to the best of 
our knowledge, this is the
closest competitor for \ExtremeHunter{} and we
empirically compare our algorithm to \ThresholdAscent{} in
Section~\ref{sec:exp}.

In this paper we also target the extreme bandit setting, but contrary to
\cite{cicirello2005max,streeter2006asymptotically,streeter2006simple},
we only make a semi-parametric assumption on the distribution; the
second order Pareto assumption (Definition~\ref{def:sop}), which is standard
 in extreme value theory (see
e.g.,~\cite{hall1984best,dehaan2006extreme}). This is significantly weaker than
the parametric assumptions made in the prior works for extreme bandits.
Furthermore, we provide a \textit{finite-time} regret bound for our more
\textit{general semi-parametric setting} (Theorem~\ref{thm:extremehunter}),
while the
prior works only offer asymptotic results. In particular, we
provide an upper bound on the rate at which the regret becomes negligible when
compared to the optimal oracle reward (Definition~\ref{def:extremeregret}).


\section{Extreme Hunter}
\label{sec:algo}
In this section, we present our main results. In particular, we present the algorithm and the main theorem that bounds its
extreme regret. Before that, we first provide an initial
result on the expectation of the maximum of second order Pareto random
variables  which will set the benchmark for the oracle regret.
We first characterize the expectation of the maximum of second order Pareto
distributions. The following lemma states that the expectation of the maximum 
of i.i.d.\,second order Pareto samples is equal, up to a negligible term, to the
expectation of the maximum of i.i.d.~Pareto samples. This result is crucial
for assessing the benchmark for the regret, in particular the expected value of
the maximal oracle sample. Theorem~\ref{cor:mean} is based on
Lemma~\ref{lem:lardev}, both provided in the appendix.
\renewcommand{\baselinestretch}{1.2} 
\begin{theorem}\label{cor:mean}
Let $X_1, \ldots, X_n$ be $n$ i.i.d.\,samples drawn according to
$(\alpha,\beta,C,C')$-second order Pareto distribution $P$ (see Definition~\ref{def:sop}). If $\alpha>1$, then:
\begin{align*}
 \left|  \mathbb E (\max_i X_i) -   
(nC)^{1/\alpha}\Gamma\left(1\!-\!\tfrac{1}{\alpha}\right) 
\right|
\leq \tfrac{4D_2}{n}(nC)^{1/\alpha}  +
\tfrac{2C'D_{\beta+1} }{C^{\beta+1}n^{\beta}}(nC)^{1/\alpha} +  B
 = o\left(\left(nC\right)^{1/\alpha}\right),
\end{align*}
where $D_2,D_{1+\beta}>0$ are some universal constants, and $B$ is
defined in the appendix~\eqref{eq:B}.
\end{theorem}
\renewcommand{\baselinestretch}{1} 

Theorem~\ref{cor:mean} implies that the optimal strategy in hindsight
attains the following expected reward:
\begin{align*}
\mathbb E \left[ G_n^*\right] \approx \max_k\left[ \left(C_k
n\right)^{1/\alpha_k}  \Gamma\left(1\!-\!\tfrac{1}{\alpha}\right)  \right]%
\quad 
\end{align*}
 \begin{wrapfigure}{r}{0.5\textwidth}
 \vspace{-2.6em}
\begin{minipage}{0.5\textwidth}
   \begin{algorithm}[H]
 \caption{\ExtremeHunter{}}
 \label{alg:ExtremeHunter}
\begin{algorithmic}
 \STATE {\bfseries Input:}
 \STATE \quad  $K$: number of arms  
 \STATE \quad $n$: time horizon
\STATE \quad  $b$: where for all $k\leq K: b \leq \beta_k$
\STATE \quad $N$: minimum number of pulls of each arm
  \STATE {\bfseries Initialize:}
 \STATE \quad $T_k \gets 0$ for all $k\leq K$
 \STATE \quad  $\delta \gets \exp(-\log^2 n)/(2nK)$ 
\STATE {\bfseries Run:}
 \FOR{ $t=1$ {\bfseries to} $n$}
 \FOR{ $k=1$ {\bfseries to} $K$}
  \IF{$T_k \le N $ }
   \STATE $B_{k,t} \gets  \infty$
   \ELSE
  \STATE estimate $\widehat h_{k,t}$ that verifies~\eqref{eq:UCBalpha}
  \STATE estimate $\widehat C_{k,t}$ using~\eqref{eq:hatCkt}
  \STATE update $B_{k,t}$ using~\eqref{eq:UCBBkt}
with~\eqref{eq:UCBalpha} and~\eqref{eq:UCBCkt}
  \ENDIF
  \ENDFOR
   \STATE Play arm $k_t \gets  \argmax_k B_{k,t}$
  \STATE $T_{k_t} \gets T_{k_t} +1$
 \ENDFOR
 \end{algorithmic}
\end{algorithm}
    \end{minipage}
 \vspace{-2em}
\end{wrapfigure}
 Our objective is therefore to find a learner $\pi$
such that $\mathbb E\left[ G_n^*\right] - \mathbb E \left[G^\pi_n\right]$ is 
negligible when compared
to $\mathbb E [ G_n^*]$, i.e.,~when compared to
$(nC^*)^{1/\alpha^*}\Gamma\left(1\!-\!\tfrac{1}{\alpha^*}\right)  \approx 
n^{1/\alpha^*}$ where
$*$ is the optimal arm.

From the discussion above, we know that the minimization of the extreme regret is
linked with the identification of the arm with the heaviest tail. Our \ExtremeHunter{}
algorithm is based on a classical idea in bandit theory: \textit{optimism in the 
face of
uncertainty}. Our strategy is to estimate $\mathbb E \left[\max_{t \leq n} 
X_{k,t}\right]$
for any $k$ and to pull the arm which maximizes its upper bound. From
Definition~\ref{def:sop}, the estimation of this quantity relies heavily on an
efficient estimation of $\alpha_k$ and $C_k$, and on associated confidence
widths. This topic is a classic problem in extreme value
theory, and such estimators exist provided that one knows a lower bound $b$ on
$\beta_k$ \cite{dehaan2006extreme, carpentier2013adaptive,carpentier2013honest}. From now on we
assume that a constant $b>0$ such that $b \leq \min_k \beta_k$ is known to the
learner. As we argue in Remark~\ref{remark:beta}, this assumption is necessary.

Since our main theoretical result is a \textit{finite-time} upper bound, in the
following exposition we carefully describe all the constants and stress what
quantities they depend on. Let $T_{k,t}$ be the number of samples drawn from arm
$k$ at time $t$. Define $\delta = \exp(-\log^2 n)/(2nK)$ and consider an
estimator $\widehat h_{k,t}$ of $1/\alpha_k$ at time $t$ that verifies the 
following condition with
probability $1-\delta$, for $T_{k,t}$ larger than some constant $N_2$ that
depends only on
$\alpha_k, C_k, C'$ and $b$:
\begin{align}\label{eq:UCBalpha}
\left|\tfrac{1}{\alpha_k} - \widehat h_{k,t}\right| \leq D \sqrt{\log(1/\delta)}
T_{k,t}^{-b/(2b+1)} = B_1(T_{k,t}),
\end{align}
where $D$ is a constant that also depends only on $\alpha_k, C_k, C',$ and $b$. 
For
instance, the estimator in~\cite{carpentier2013adaptive} (Theorem
3.7) verifies this property and provides $D$ and $N_2$ but other estimators are
possible. Consider the associated estimator for $C_k$:
\begin{align}\label{eq:hatCkt}
\widehat C_{k,t} =
T_{k,t}^{1/(2b+1)}\left(\frac{1}{T_{k,t}}\sum_{u=1}^{T_{k,t}} \mathbf
1\Big\{X_{k,u} \geq T_{k,t}^{\widehat h_{k,t}/(2b +1)}\Big\}\right)
\end{align}

For this estimator, we know \cite{carpentier2013honest} with probability 
$1-\delta$ that for $T_{k,t}\geq N_2$:
\begin{align}\label{eq:UCBCkt}
\left|C_k- \widehat C_{k,t}\right| \leq E \sqrt{\log(T_{k,t}/\delta)}
\log(T_{k,t})T_{k,T}^{-b/(2b+1)} = B_2(T_{k,t}),
\end{align}
where $E$ is derived in~\cite{carpentier2013honest} in the proof of Theorem 2.
Let $N = \max\left(A \log(n)^{2(2b+1)/b}, N_2\right)$ where $A$ depends on
$(\alpha_k,
C_k)_k, b, D, E$, and $C'$, and
is such that: $$\max\left(2B_1(N), 2B_2(N)/C_k\right) \leq 1, \hspace{2mm}
N \geq (2D\log^2 n)^{(2b+1)/b},
\hspace{2mm}
  N >
\left(\tfrac{2D \sqrt{\log(n)^2}}{1-\max_k 1/\alpha_k}
\right)^{(2b+1)/b}\!\!\!$$
This inspires Algorithm~\ref{alg:ExtremeHunter},
which
first pulls each arm $N$ times and then, at each time $t>KN$, pulls the arm that
maximizes $B_{k,t}$, which we define as:
\begin{align}\label{eq:UCBBkt}
  \left( \left(\widehat C_{k,t}+B_2\left(T_{k,t}\right)\right)  
n\right)^{\widehat h_{k,t} +
B_1(T_{k,t})} \bar \Gamma\left(\widehat h_{k,t},B_1\left(T_{k,t}\right)\right),
\end{align}
where $\bar \Gamma(x,y) = \tilde \Gamma(1-x-y)$, where
we set $\tilde\Gamma=\Gamma$ for any $x > 0$ and $+\infty$ otherwise.

\renewcommand{\baselinestretch}{1.2} 
\begin{remark}\label{remark:beta}
A natural question is whether it is possible to learn $\beta_k$ as well.
In fact, this is not possible for this model and a negative result
was proved by~\cite{carpentier2013honest}. The result states that in 
this setting it is
not possible to test between two fixed values of
$\beta$ uniformly over the set of distributions. Thereupon, we define $b$ as a
lower bound
for all $\beta_k$. With regards to the Pareto
distribution, $\beta = \infty$ corresponds
to the exact Pareto distribution, while
$\beta = 0$ for such distribution
that is not (asymptotically) Pareto.
\end{remark}
\renewcommand{\baselinestretch}{1} 
We show that this algorithm meets the desired properties.
The following theorem states our main result by upper-bounding the extreme
regret of \ExtremeHunter{}.

\renewcommand{\baselinestretch}{1.2} 
\begin{theorem}\label{thm:extremehunter}
Assume that the distributions of the arms are respectively $(\alpha_k,\beta_k,C_k,C')$ second order 
Pareto (see Definition~\ref{def:sop}) with $\min_k \alpha_k>1$. If $n \geq Q$, 
the expected extreme regret of \ExtremeHunter{} is
bounded from above as:
\begin{align*} 
\mathbb E [G_n^*] o(1) = L(nC^*)^{1/\alpha^*}
\left(  \tfrac{K}{n}\log(n)^{(2b+1)/b}
+ n^{-\log(n)(1-1/\alpha^*)}
+ n^{-b/((b+1)\alpha^*)}  \right),
\end{align*}
where $L,Q>0$ are some constants depending only on $(\alpha_k, C_k)_k, C',$ and 
 $b$
(Section~\ref{sec:analysis}).
\end{theorem}
\renewcommand{\baselinestretch}{1} 

Theorem~\ref{thm:extremehunter} states that the
\ExtremeHunter{} strategy performs almost as well as the best (oracle) strategy,
up to a term that is negligible when compared to the performance of the oracle
strategy. Indeed, the regret is negligible when compared to
$(nC^*)^{1/\alpha^*}$, which is the order of magnitude of the performance of the
best oracle strategy $\mathbb E \left[ G_n^*\right] =  \max_{k\leq K}\mathbb E 
\left[\max_{t\leq
n}X_{k,t}\right]$. Our algorithm thus detects the arm that has the heaviest 
tail.

For $n$ large enough (as a function of $(\alpha_k,\beta_k,C_k)_k,C'$ and $K$),
the two first terms in the regret become negligible
when compared to the third one, and the regret is then bounded as:
\[
\mathbb E \left[R_n\right] \leq \mathbb  E \left[G_n^*\right] 
 \mathcal{O}\left(n^{-b/((b+1)\alpha^*)}\right) 
\]
We make two observations: First, the larger the $b$, the tighter
this bound is, since the model is then closer to the parametric case. Second, 
smaller $\alpha^*$ also tightens the bound, since the best arm is
then very heavy tailed and much easier to recognize.


\section{Analysis}
\label{sec:analysis}

In this section, we prove an upper bound
on the extreme regret of Algorithm~\ref{alg:ExtremeHunter}
stated in Theorem~\ref{thm:extremehunter}.
Before providing the detailed proof, we give a high-level overview
and the intuitions. 

In \textit{Step 1}, we define the (favorable)  high
probability event $\xi$ of interest, useful for analyzing the mechanism of the
bandit algorithm. In \textit{Step 2}, given $\xi$, we bound the estimates of 
$\alpha_k$ 
and $C_k$, and use them to bound the main upper confidence bound. 
In \textit{Step 3}, we upper-bound the number of pulls of each suboptimal arm:
we prove that with high probability we do not pull them too often. 
This enables us to guarantee that the number of pulls of the optimal
arms $*$ is on $\xi$ equal to $n$ up to a negligible term. 

The final 
\textit{Step 4} of the proof is concerned with using this lower bound on the 
number of pulls of the optimal arm in order to lower bound the expectation of 
the
maximum of the collected samples. Such step is typically straightforward
in the classical (mean-optimizing) bandits by the linearity of the expectation. 
It is not
straightforward in our setting. We therefore prove Lemma~\ref{lem:exp},
in which we show that the expected value of the maximum of the samples
in the favorable event $\xi$ will be not too far away from the one that
we obtain without conditioning on $\xi$.
\paragraph{Step 1: High probability event.}
In this step, we define the favorable event $\xi$.  We set 
$\delta~\eqdef~\exp(-\log^2\!n)/(2nK)$ and consider the event $\xi$ such that 
for
any $k\leq K$, $N \leq T \leq n$:
\begin{align*}
\left|\tfrac{1}{\alpha_k} - \tilde h_{k}(T)\right| &\leq D \sqrt{\log(1/\delta)}
T^{-b/(2b+1)}, \\
\left|C_k- \tilde C_{k}(T)\right| &\leq E \sqrt{
\log(T/\delta)}T^{-b/(2b+1)},
\end{align*}
where $\tilde h_{k}(T)$ and $\tilde C_{k}(T)$ are the estimates of  $1/\alpha_k$
and $C_k$ respectively using the first $T$ samples. Notice, they are not the
same as $\widehat h_{k,t}$ and $\widehat C_{k,t}$ which are the estimates of 
the 
same
quantities at time $t$ for the algorithm, and thus with $T_{k,t}$ samples.
The probability of $\xi$ is larger than
$1 - 2nK\delta$ by a union bound on~\eqref{eq:UCBalpha} and
 \eqref{eq:UCBCkt}.
\paragraph{Step 2: Bound on $B_{k,t}$.}
The following lemma holds on $\xi$ for upper- and lower-bounding $B_{k,t}$.
\begin{lemma}{(proved in the appendix)}\label{bound:Bkt}
 On $\xi$, we have that for any $k \leq K$, and for $T_{k,t} \geq N$:
 \begin{align}\label{eq:step1finish}
\left( C_k n\right)^{\frac{1}{\alpha_k}} 
\Gamma\left(1\!-\!\tfrac{1}{\alpha_k}\right)
\leq B_{k,t}  
\!\leq\!(C_{k}n)^{\frac{1}{\alpha_{k}}}
\Gamma\left(1\!-\!\tfrac{1}{\alpha_k}\right)
\left(1+F\log(n)\sqrt{\log(n/\delta)}T_{k,t}^{-b/(2b+1)}
\right)
\end{align}
\end{lemma}
\paragraph{Step 3: Upper bound on the number of pulls of a suboptimal arm.}
We proceed by using the bounds on  $B_{k,t}$ from the previous step to
upper-bound the number of suboptimal pulls.
Let $*$ be the best arm. Assume that at round $t$, some arm $k\neq*$ is pulled.
Then
by definition of the algorithm
$B_{*,t}\leq B_{k,t},$
which implies by Lemma~\ref{bound:Bkt}:
\begin{align*}
\left(C^*n\right)^{1/\alpha^*}\Gamma\left(1\!-\!\tfrac{1}{\alpha^*}\right) 
\leq 
\left(C_{k}n\right)^{1/\alpha_{k}}\Gamma\left(1\!-\!\tfrac{1}{\alpha_k}\right) 
\left(1+F\log(n)\sqrt{\log(n/\delta)}T_{k,t}^{-b/(2b+1)}
\right)
\end{align*}
Rearranging the terms we get:
\begin{align}\label{eq:step3a}
 \frac{\left(C^*n\right)^{1/\alpha^*}
\Gamma\left(1\!-\!\frac{1}{\alpha^*}\right)
}{\left(C_{k}n\right)^{1/\alpha_{k}}
\Gamma\big(1\!-\!\frac{1}{\alpha_k}\big)}
\leq 1+F\log(n)\sqrt{\log(n/\delta)}T_{k,t}^{-b/(2b+1)}
\end{align}
We now define $\Delta_k$ which is analogous to the \textit{gap}
in the classical bandits:
\[
\Delta_k =  
\frac
{\left(C^*n\right)^{1/\alpha^*}
\Gamma\left(1\!-\!\frac{1}{\alpha^*}\right)}
{\left(C_{k}n\right)^{1/\alpha_{k}}
\Gamma\big(1\!-\!\frac{1}{\alpha_k}\big)}
 - 1
\]
Since $T_{k,t}\leq n$, \eqref{eq:step3a} implies
for some problem dependent constants $G$ and $G'$ dependent only on 
$(\alpha_k, C_k)_k, C'$ and $b$, but independent of
$\delta$ that:
\begin{align*}
T_{k,t}\leq N + 
G'\left(\tfrac{\log^2\!n\log(n/\delta)}{\Delta_k^2}\right)^{(2b+1)/(2b)} 
\leq N +
G\left(\log^2\!n\log(n/\delta)\right)^{(2b+1)(2b)}
\end{align*}
This implies that number $T^*$ of pulls of arm $*$ 
is  with probability $1-\delta'$, at least
$$n - \sum_{k\neq *}
G\left(\log^2\!n\log(2nK/\delta')\right)^{(2b+1)/(2b)} -
KN,$$
where
$\delta' = 2 nK\delta $.
Since $n$ is larger than
\[
Q \geq 2KN + 
2 GK\left(\log^2\!n\log\left(2nK/\delta'\right)\right)^{(2b+1)/(2b)}  
,\] 
we have that $T^* \geq \frac n2$ as a corollary.
\paragraph{Step 4: Bound on the expectation.}
We start by lower-bounding the expected gain:
\begin{align*}
\mathbb E [G_n] \! = \! \mathbb E \left[\max_{t\leq n} X_{I_t,T_{k,t}}\right]
\!\geq\!  \mathbb E \left[\max_{t\leq n} X_{I_t,T_{k,t}}\mathbf 1\{\xi\}\right]
\!\geq\!  \mathbb E \left[\max_{t\leq n} X_{*,T_{*,t}}\mathbf 1\{\xi\}\right]
\!= \! \mathbb E \left[\max_{i\leq T^*} X_{i}\mathbf 1\{\xi\}\right]
\end{align*}
The next lemma links the expectation of $\max_{t\leq T^*} X_{*,t}$ with the
expectation of $\max_{t\leq T^*} X_{*,t}\mathbf 1\{\xi\}$.

\renewcommand{\baselinestretch}{1.2} 
\begin{lemma}{(proved in the appendix)}\label{lem:exp}
Let $X_1, \ldots, X_T$ be i.i.d.\,samples from an
($\alpha, \beta, C, C'$)-second
order Pareto distribution $F$. Let $\xi'$ be an event of
probability larger than $1-\delta$. Then for $\delta < 1/2$ and for $T\ge Q$ 
large 
enough so that
$c\max\left(1/T, 1/T^{\beta}\right) \leq 1/4$ for a given constant
$c>0$, that depends only on $C,C'$ and $\beta$, and also for $T\geq
\log(2)\max\big(C\left(2C'\right)^{1/\beta}, 8\log\left(2\right)\big)$:
\begin{align*}
\!\mathbb E \left[\max_{t\leq T} X_{t}\mathbf 1\{\xi\}\right]
&\geq\left(TC\right)^{1/\alpha} \Gamma\left(1\!-\!\tfrac{1}{\alpha}\right) 
- \big(4+\tfrac{8}{\alpha-1}\big)\left(TC\right)^{1/\alpha} 
\delta^{1\!-\!1/\alpha}\\
&\qquad\qquad - 2\left(\tfrac{4D_2 
}{T}\left(TC\right)^{1/\alpha}  +
\tfrac{2C'D_{1+\beta} }{C^{1+\beta}T^{\beta}}\left(TC\right)^{1/\alpha} +  
B\right).
\end{align*}
\end{lemma}
\renewcommand{\baselinestretch}{1} 
Since $n$ is large enough so that 
$2n^2K\delta' =
2n^2K\exp\left(-\log^2\!n\right)\leq 1/2$, where $\delta' = 
\exp\left(-\log^2\!n\right),$ and
 the probability of $\xi$ is larger than
$1-\delta'$,
we can use Lemma~\ref{lem:exp} for the optimal arm:
\begin{align*}
\mathbb E \left[\max_{t\leq T^*} 
X_{*,t}\mathbf 1\{\xi\}\right]\!
\!\geq\! \left(T^*C^*\right)^{\frac{1}{\alpha^*}}\! \left[
\Gamma\!\left(1\!-\!\tfrac{1}{\alpha^*}\right)  \!-\! 
\big(4\!+\!\tfrac{8}{\alpha-1}\big)
\delta'^{1-\frac{1}{\alpha^*}}
 \!-\! \tfrac{8D_2 }{T^*}  
\!-\!
\tfrac{4C'D_{\max} }{(C^*)^{1+b}(T^*)^{b}} \!-\!  
\tfrac{2B}{\left(T^*C^*\right)^{\frac{1}{\alpha^*}}}\right]\!, 
\end{align*}
where $D_{\max} \eqdef \max_i D_{1+\beta_i}$. Using 
Step~3, we bound the above with a function of $n$.
In particular, we lower-bound the last three terms in the brackets using $T^* 
\geq \frac n2$ 
and the $(T^*C^*)^{1/\alpha^*}$ factor  as:
$$(T^*C^*)^{1/\alpha^*} \geq
(nC^*)^{1/\alpha^*} \Big(1 - 
\tfrac{GK}{n}
\left(\log(2n^2K/\delta')\right)^{\frac{2b+1}{2b}}\! -\!
\tfrac{KN}{n}\Big)$$
We are now ready to relate the lower bound on the gain of \ExtremeHunter{}
with the upper bound of the gain of the optimal policy 
(Theorem~\ref{cor:mean}), which brings us the upper bound for the regret:
\begin{align*}
\mathbb E \left[R_n\right] &= 
\mathbb E\left[G_n^*\right] - \mathbb E\left[G_n\right]
\leq \mathbb E\left[G_n^*\right] - \mathbb E \left[\max_{i\leq T^*} X_{i}\right]
\leq  \mathbb E\left[G_n^*\right] - 
\mathbb E \left[\max_{t\leq T^*} 
X_{*,t}\mathbf 1\{\xi\}\right] \leq
\\ &\leq
H(nC^*)^{1/\alpha^*}\! \left(\tfrac{1}{n}\!  +\!  \tfrac{1}{(nC^*)^{b}}
\!+\!
\tfrac{GK}{n}
\left(\log(2n^2K/\delta')\right)^{\frac{2b+1}{2b}}
+ \tfrac{KN}{n} +  \delta'^{1-1/\alpha^*} + 
\tfrac{B}{(nC^*)^{1/\alpha^*}}\right),
\end{align*}
where $H$ is a constant that
depends on $(\alpha_k, C_k)_k,C',$ and $b$. 
To bound the last term, we use the definition of $B$~\eqref{eq:B} to get the 
$n^{- \beta^*/((\beta^*+1)\alpha^*)}$ term, upper-bounded by  
$n^{-b/((b+1)\alpha^*)}$ 
as $b\leq \beta^*$. 
Notice that this final term also eats up $n^{-1}$ and $n^{-b}$ terms 
since $b/((b+1)\alpha^*) \leq \min(1, b)$. 

We finish by using $\delta' = 
\exp\left(-\log^2\!n\right)$ and grouping the problem-dependent constants into 
$L$ to get the final upper bound, which is on the order of $\mathbb E 
\left[G_n^*\right] o(1)$:
\begin{align*}
\mathbb E \left[R_n\right] \leq   L(nC^*)^{1/\alpha^*}
\left( \tfrac{K}{n}\log(n)^{(2b+1)/b}
+ n^{-\log(n)(1-1/\alpha^*)}
+ n^{-b/((b+1)\alpha^*)} \right) 
\end{align*}





\vspace{-0.5em}
\section{Experiments}\label{sec:exp}
 \vspace{-0.5em}
In this section, we empirically evaluate 
\ExtremeHunter{} on synthetic and real-world data.
The measure of our evaluation is the
\textit{extreme regret} from Definition~\ref{def:extremeregret}.
Notice that even though we evaluate
the regret as a function of time $T$, the extreme regret is \textit{not
cumulative} and it is more in the spirit of \textit{simple 
regret}~\cite{bubeck2009pure}.
We compare our \ExtremeHunter{} with  
\ThresholdAscent{}~\cite{streeter2006simple}. Moreover, we also
compare to classical \UCB{}~\cite{auer2002finite},
as an example of the algorithm that aims for the arm with the \textit{highest
mean} as opposed to the \textit{heaviest tail}.
When the distribution of a single arm has
both the highest mean and the heaviest-tail,
both \ExtremeHunter{} and \UCB{} are expected to perform the same
with respect to the extreme regret.
In the light of Remark~\ref{remark:beta}, we set $b = 1$ to consider a wide
class of distributions. 
%

\begin{figure}
\begin{center}
\includegraphics[width=0.32\columnwidth]{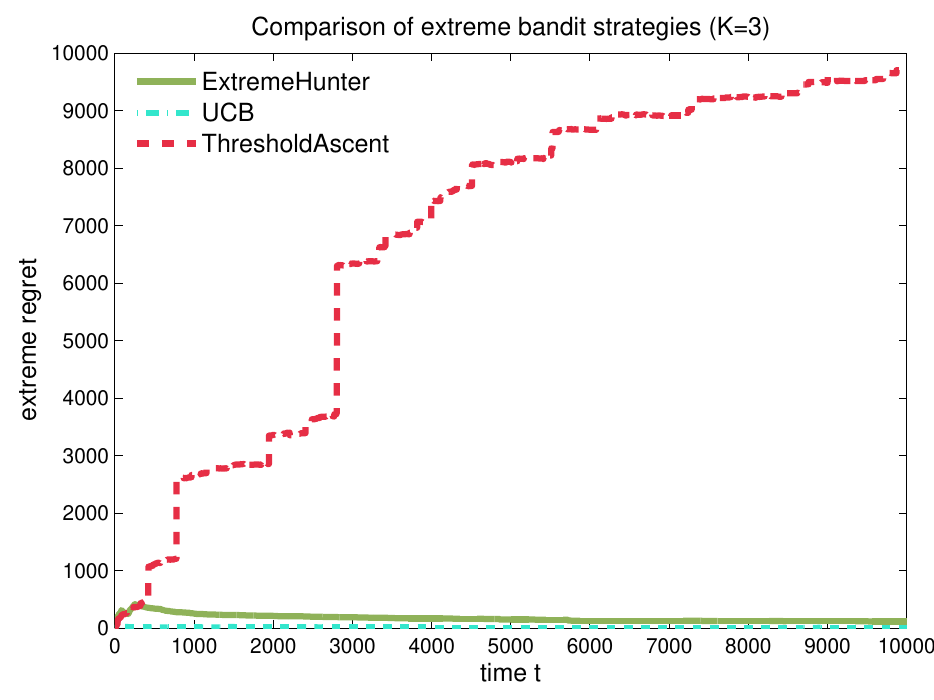}
\includegraphics[width=0.32\columnwidth]{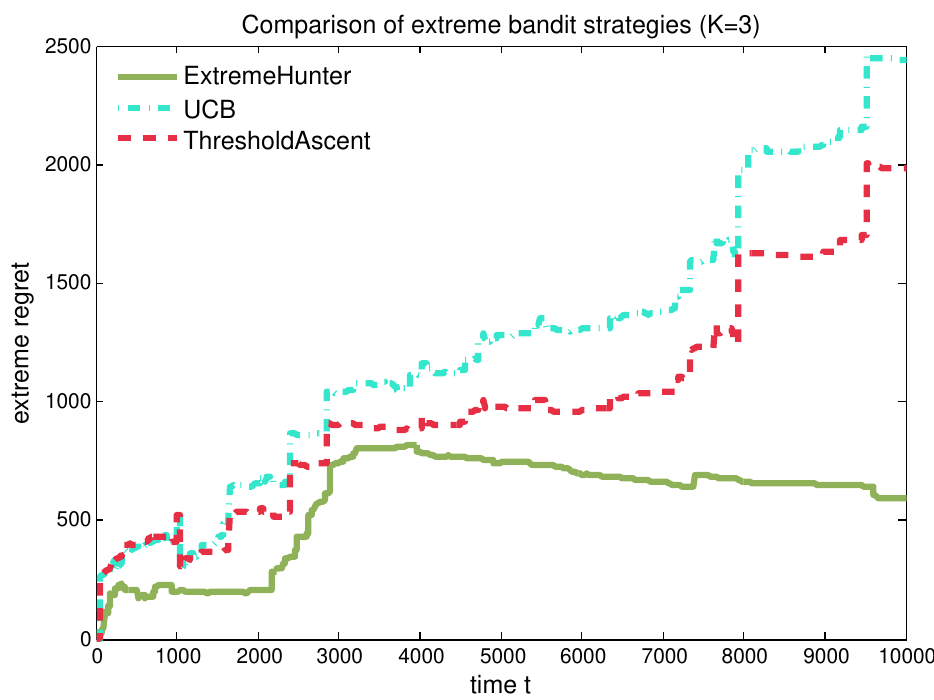}
\includegraphics[width=0.32\columnwidth]{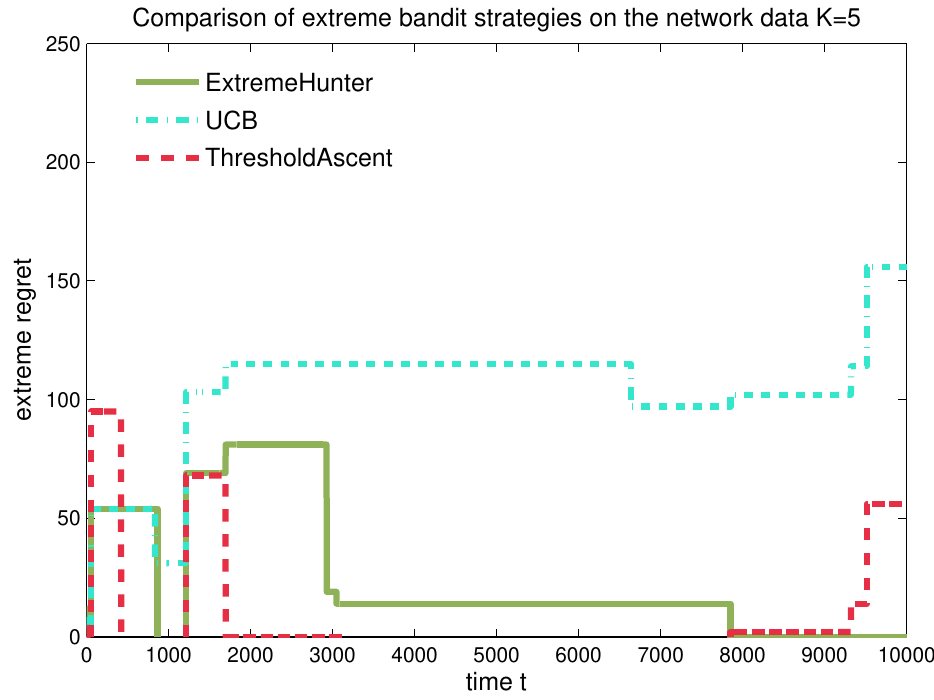}
\caption{Extreme regret as a function of time for the exact Pareto 
distributions (left), approximate
Pareto (middle) distributions,
 and the network traffic data (right).}
\vspace{-1em}
\label{fig:experiments}
\end{center}
\end{figure}
\paragraph{Exact Pareto Distributions}
In the first experiment, we consider
$K = 3$ arms with the distributions
$P_k(x) = 1 - x^{-\alpha_k}$, where
$\alpha = [5,1.1,2]$.
Therefore, the most heavy-tailed distribution
is associated with the arm $k=2$.
Figure~\ref{fig:experiments} (left) displays
the averaged result of 1000 simulations
with the time horizon $T = 10^4$.
We observe that \ExtremeHunter{}
eventually keeps allocating most of the pulls
to the arm of the interest.
Since in this case, the arm with the heaviest tail
is also the arm with the largest mean, \UCB{} also performs well and it is 
even able to detect the best arm earlier. \ThresholdAscent{},
on the other hand, was not always able to allocate the pulls
properly in $10^4$ steps. This may be due to the
discretization of the rewards that this algorithm is using.
\paragraph{Approximate Pareto Distributions}
For the exact Pareto distributions,
the smaller the tail index the higher
the mean and even  UCB obtains a good performance.
However, this is no longer
necessarily the case for the approximate
Pareto distributions.
For this purpose, we perform
the second experiment
where we mix an exact Pareto distribution
with a Dirac distribution in 0.
We consider $K=3$ arms.
Two of the arms follow the exact Pareto distributions with
$\alpha_1 = 1.5$ and $\alpha_3 = 3$.
 On the other hand, the second arm has a mixture weight of 
0.2 for the exact Pareto distribution with $\alpha_2 = 1.1$
and 0.8 mixture weight of the Dirac distribution in 0.
For this setting, the second arm is the most heavy-tailed
but the first arm has the largest mean.
Figure~\ref{fig:experiments} (middle) shows the result.
We see that \UCB{} performs worse since it eventually focuses
on the arm with the largest mean.  \ThresholdAscent{}
performs better than \UCB{} but not as good as \ExtremeHunter{}.
\paragraph{Computer Network Traffic Data}
In this experiment, we evaluate \ExtremeHunter{} on  heavy-tailed network
traffic data which was collected from  user laptops in the
enterprise environment~\cite{agosta2013mixture}.
The objective is to allocate the sampling capacity among the computer nodes 
(arms),
in order to find the largest outbursts of the network activity.
This information then serves an IT department to
further investigate the source of the extreme network traffic.
For each arm, a sample at the time $t$ corresponds to
the number of network activity events for 4 consecutive seconds.
Specifically, the network events are the starting times of packet flows.  In 
this experiment, we selected \mbox{$K=5$} laptops (arms),
where the recorded sequences were long enough.
Figure~\ref{fig:experiments} (right) shows that 
\ExtremeHunter{} again outperforms both \ThresholdAscent{} and \UCB{}.


\paragraph{Acknowledgements}
\label{sec:Acknowledgements}
We would like to thank John Mark Agosta and Jennifer Healey for the network 
traffic data. The research presented in this paper was supported by Intel 
Corporation, by French Ministry of Higher Education and Research, and by 
European Community's Seventh Framework Programme (FP7/2007-2013) under grant 
agreement n$^{\rm o}$270327 (CompLACS).

 \bibliography{library}

\begin{thebibliography}{10}

\bibitem{abe2006outlier}
Naoki Abe, Bianca Zadrozny, and John Langford.
\newblock {Outlier Detection by Active Learning}.
\newblock In {\em ACM SIGKDD International Conference on Knowledge Discovery and Data Mining (KDD)}, pages 504--509, 2006.

\bibitem{agosta2013mixture}
John~Mark Agosta, Jaideep Chandrashekar, Mark Crovella, Nina Taft, and Daniel
  Ting.
\newblock {Mixture models of endhost network traffic}.
\newblock In {\em IEEE International Conference on Computer Communications (INFOCOM),}, pages 225--229.

\bibitem{auer2002finite}
Peter Auer, Nicol\`{o} Cesa-Bianchi, and Paul Fischer.
\newblock {Finite-time Analysis of the Multiarmed Bandit Problem}.
\newblock {\em Machine Learning}, 47(2-3):235--256, 2002.

\bibitem{bubeck2013bandits}
S\'{e}bastien Bubeck, Nicol\`{o} Cesa-Bianchi, and G\'{a}bor Lugosi.
\newblock {Bandits With Heavy Tail}.
\newblock {\em IEEE Transactions on Information Theory}, 59(11):7711--7717,
  2013.

\bibitem{bubeck2009pure}
S\'{e}bastien Bubeck, R\'{e}mi Munos, and Gilles Stoltz.
\newblock {Pure Exploration in Multi-armed Bandits Problems}.
\newblock {\em Algorithmic Learning Theory}, pages 23--37, 2009.

\bibitem{carpentier2013adaptive}
Alexandra Carpentier and Arlene K.~H. Kim.
\newblock {Adaptive and minimax optimal estimation of the tail coefficient}.
\newblock {\em Statistica Sinica}, 2014.

\bibitem{carpentier2013honest}
Alexandra Carpentier and Arlene K.~H. Kim.
\newblock {Honest and adaptive confidence interval for the tail coefficient in
  the Pareto model}.
\newblock {\em Electronic Journal of Statistics}, 2014.

\bibitem{chandola2009anomaly}
Varun Chandola, Arindam Banerjee, and Vipin Kumar.
\newblock {Anomaly detection: A survey}.
\newblock {\em ACM Computing Surveys}, 41(3):15:1--15:58, July 2009.

\bibitem{cicirello2005max}
Vincent~A. Cicirello and Stephen~F. Smith.
\newblock {The max k-armed bandit: A new model of exploration applied to search
  heuristic selection}.
\newblock {\em AAAI Conference on Artificial Intelligence (AAAI)}, 2005.

\bibitem{dehaan2006extreme}
Laurens de~Haan and Ana Ferreira.
\newblock {\em {Extreme Value Theory: An Introduction}}.
\newblock Springer Series in Operations Research and Financial Engineering.
  Springer, 2006.

\bibitem{fisher1928limiting}
Ronald~Aylmer Fisher and Leonard Henry~Caleb Tippett.
\newblock {Limiting forms of the frequency distribution of the largest or
  smallest member of a sample}.
\newblock {\em Mathematical Proceedings of the Cambridge Philosophical
  Society}, 24:180, 1928.

\bibitem{gnedenko1943distribution}
Boris Gnedenko.
\newblock {Sur la distribution limite du terme maximum d'une s\'{e}rie
  al\'{e}atoire}.
\newblock {\em The Annals of Mathematics}, 44(3):423--453, 1943.

\bibitem{hall1984best}
Peter Hall and Alan~H. Welsh.
\newblock {Best Attainable Rates of Convergence for Estimates of Parameters of
  Regular Variation}.
\newblock {\em The Annals of Statistics}, 12(3):1079--1084, 1984.

\bibitem{lai1985asymptotically}
Tze~L. Lai and Herbert Robbins.
\newblock {Asymptotically efficient adaptive allocation rules}.
\newblock {\em Advances in Applied Mathematics}, 6(1):4--22, 1985.

\bibitem{liu2011dynamic}
Keqin Liu and Qing Zhao.
\newblock {Dynamic Intrusion Detection in Resource-Constrained Cyber Networks}.
\newblock In {\em IEEE International Symposium on Information Theory (ISIT)}, 2012.

\bibitem{markou2003novelty}
Markos Markou and Sameer Singh.
\newblock {Novelty detection: a review, part 1: statistical approaches}.
\newblock {\em Signal Processing}, 83(12):2481--2497, 2003.

\bibitem{neill2010multivariate}
Daniel~B. Neill and Gregory~F. Cooper.
\newblock {A multivariate Bayesian scan statistic for early event detection and
  characterization}.
\newblock {\em Machine Learning}, 79:261--282, 2010.

\bibitem{priebe2005scan}
Carey~E. Priebe, John~M. Conroy, David~J. Marchette, and Youngser Park.
\newblock {Scan Statistics on Enron Graphs}.
\newblock In {\em Computational and Mathematical Organization Theory},
  volume~11, pages 229--247, 2005.

\bibitem{steinwart2005classification}
Ingo Steinwart, Don Hush, and Clint Scovel.
\newblock {A Classification Framework for Anomaly Detection}.
\newblock {\em Journal of Machine Learning Research}, 6:211--232, 2005.

\bibitem{streeter2006simple}
Matthew~J. Streeter and Stephen~F. Smith.
\newblock {A Simple Distribution-Free Approach to the Max k-Armed Bandit
  Problem.}
\newblock In {\em International Conference on Principles and Practice of Constraint Programming (CP)}, volume
  4204, pages 560--574, 2006.

\bibitem{streeter2006asymptotically}
Matthew~J. Streeter and Stephen~F. Smith.
\newblock {An Asymptotically Optimal Algorithm for the Max k-Armed Bandit
  Problem}.
\newblock In {\em AAAI Conference on Artificial Intelligence (AAAI)},
  pages 135--142, 2006.

\bibitem{turner2010fast}
Ryan Turner, Zoubin Ghahramani, and Steven Bottone.
\newblock {Fast online anomaly detection using scan statistics}.
\newblock {\em IEEE Workshop on Machine Learning for Signal Processing (MLSP)}, 2010.

\end{thebibliography}
 \bibliographystyle{plain}

\appendix
\section{Proof of Lemma 3}
\begin{lemma}
 \label{lem:lardev}
Assume that $X_1, \ldots, X_T$ are  $T$ i.i.d.~samples drawn according to
$(\alpha,\beta,C,C')$-second order Pareto distribution,
then for any $x \geq B$:
\begin{align}
&\!\left|\mathbb P\left(\max_i X_i \leq x\right)\! -\!
\exp\left(\!-TCx^{-\alpha}\right)\right|
\leq \!M\! \exp\left(\!-TCx^{-\alpha}\right) \nonumber \\
 & \mbox{ where }   M = \frac{4}{T}\left(TCx^{-\alpha}\right)^2
+\frac{2C'}{C^{\beta+1}T^{\beta}} \left(TCx^{-\alpha}\right)^{\beta+1},
\label{eq:M}
\end{align}
\begin{center}
where $B$ is defined as:
\end{center}
\begin{align}\label{eq:B}
B =
\max\left(\left(2C'/C\right)^{1/(\alpha\beta)},
\left(8C\right)^{1/\alpha}, (2TC')^{1/(\alpha(1+\beta))}\right).
\end{align}
Alternatively, let $u\in(0,1)$. When
$T\geq \log(1/u)B^{\alpha}/C$:
\begin{align*}
\left|\mathbb P\left( \max_i X_i \leq
\left(TC/\log\left(1/u\right)\right)^{1/\alpha}\right) - u\right|
&\leq u\left(\frac{4}{T}\log(1/u)^2
+\frac{2C'}{C^{\beta+1}T^{\beta}}(\log(1/u))^{1+\beta}\right)\\
&=u \times
\cO\left(\frac{1}{T}\log\left(1/u\right)^2+ \frac{1}{T^{\beta}}
\log\left(1/u\right)^{1+\beta} \right).
\end{align*}

\end{lemma}
\begin{proof} 
Consider $x\geq B$. Since the samples are i.i.d., we
are going to study the following quantity\footnote{Notice that $u^T$ means `$u$ to the power of $T$' and not `$u$ transposed'.}:
\begin{align}\label{eq:carac}
\mathbb P( \max_i X_i \leq x) &= P(x)^T.
\end{align} Since $P$ is a  second order Pareto,
we have for any $x \geq 0$:
\begin{align}\label{eq:secondparetoapprox}
1 - Cx^{-\alpha} - C' x^{-\alpha(1+\beta)} \leq P(x) \leq 1 - Cx^{-\alpha}
+C' x^{-\alpha(1+\beta)}.
\end{align}
Since $x \geq B$, we deduce from the first two terms in~\eqref{eq:B} that:
\begin{align}\label{eq:propertiesfromB}
Cx^{-\alpha} \geq 2C' x^{-\alpha(1+\beta)} \quad \mbox{and} \quad 2 Cx^{-\alpha}
\leq 1/4.
\end{align}
Let $c_x$ be the quantity that depends on $x$ and that is such that
$P(x) = 1 - Cc_xx^{-\alpha}$.
With such definition we know by~\eqref{eq:secondparetoapprox}
and further by the second inequality in~\eqref{eq:propertiesfromB} that:
\begin{align}\label{eq:cxbound}
\left|c_x-1\right| \leq \frac{C'x^{-\alpha\beta}}{C}
\leq 1/2.
\end{align}

Let $y = Cc_xx^{-\alpha}$.
By~\eqref{eq:propertiesfromB} and~\eqref{eq:cxbound} we get that $y \in [0,
\frac12]$. For any $y \in [0, \frac12]$, we have,
\[
-y - y^2 \leq \log(1-y)\leq -y.
\]
Taking the exponential, setting $y = Cc_xx^{-\alpha}$, and raising to the $T$-th power, we obtain:
\begin{align*}
 \exp\left(-T\left(Cc_xx^{-\alpha}\right)^2\right) \leq
\frac{\left(1-Cc_xx^{-\alpha}\right)^T}{\exp\left(-T\left(Cc_xx^{-\alpha}
\right)\right)} \leq 1,
\end{align*}
which by~\eqref{eq:carac}, the definition of $c_x$,  and both inequalities  
in~\eqref{eq:secondparetoapprox} yields:
\begin{align*}
\exp\left(-T\left(2Cx^{-\alpha}\right)^2 - T C'
x^{-\alpha\left(1+\beta\right)}\right)
\leq \frac{\mathbb P( \max_i X_i \leq x)}{\exp(-TCx^{-\alpha})}
\leq \exp\left(T C' x^{-\alpha(1+\beta)}\right).
\end{align*}
After multiplication and subtraction of $\exp(-TCx^{-\alpha})$:
\begin{align*}
&\exp\left(-TCx^{-\alpha}\right) \left( \exp\left(-4T(Cx^{-\alpha})^2 - T C'
x^{-\alpha(1+\beta)}\right) -1\right) \\
&\qquad\leq \mathbb P\left( \max_i X_i \leq x\right) - 
\exp\left(-TCx^{-\alpha}\right)\\
&\qquad\leq \exp\left(-TCx^{-\alpha}\right) \left(\exp\left(T C' 
x^{-\alpha(1+\beta)}\right) - 1\right).
\end{align*}

We will now simplify the $exp(y) - 1$ terms in the previous inequality. For any 
$y$ such that $y \in (0,1/2)$, we have $\exp(y)-1 \leq 2y$  and
for any $y \in \mathbb R$ we have that $y\leq \exp(y)-1 $.
In particular, this implies whenever $x \ge B\geq 
(2TC')^{1/(\alpha(1+\beta))}$, which is the third term in~\eqref{eq:B}:
\begin{align*}
&\exp\left(-TCx^{-\alpha}\right) \Big(-4T(Cx^{-\alpha})^2 - T C'
x^{-\alpha(1+\beta)}\Big)\\
&\qquad\leq \mathbb P\left( \max_i X_i \leq x\right) - 
\exp\left(-TCx^{-\alpha}\right)\\
&\qquad\leq \exp\left(-TCx^{-\alpha}\right) \left(2T C' 
x^{-\alpha(1+\beta)}\right).
\end{align*}
This implies that for any $x \geq B$ and $M$ as defined in~\eqref{eq:M}:
\begin{align*}
&\left|P\left(\max_i X_i \leq x\right) -
\exp\left(\!-TCx^{-\alpha}\right)\right|
\leq M \exp\left(-TCx^{-\alpha}\right)
\end{align*}

We now simply reparametrize this upper bound
by setting: $$u  = \exp(-TCx^{-\alpha}).$$
Then $u\in(0,1)$ and $x = \left(TC/\log\left(1/u\right)\right)^{1/\alpha}.$
Then $x$ is larger than $B$ as soon as $T$ is larger than 
$\log(1/u)B^{\alpha}/C$. It follows that for such
$T$, 
by the reparametrization in $u$, the rate of convergence of the
distribution to a Fr\'echet distribution:
\begin{align*}
\Big|&\mathbb P\left( \max_i X_i \leq
\left(TC/\log\left(1/u\right)\right)^{1/\alpha}\right) - u\Big|
\leq u\left(\frac{4}{T}(\log 1/u)^2
+\frac{2C'}{C^{\beta+1}T^{\beta}}\log(1/u)^{\beta+1}\right).
\end{align*}
\end{proof}

\section{Proof of Theorem 1}
\setcounter{theorem}{1-1}
\begin{theorem}
Assume that $X_1, \ldots, X_T$ are $T$ i.i.d.~samples drawn according to
$(\alpha,\beta,C,C')$-second order Pareto distribution P. If $\alpha>1$, then:
\begin{align*}
 \left|  \mathbb E (\max_i X_i) 
\!-\!(TC)^{1/\alpha}\Gamma\left(1\!-\!\tfrac{1}{\alpha}\right) 
\right|
\leq \frac{4D_2 (TC)^{1/\alpha}}{T}  +
\frac{2C'D_{\beta+1} (TC)^{1/\alpha}}{C^{\beta+1}T^{\beta}} +  B
 = o\left(\left(TC\right)^{1/\alpha}\right),
\end{align*}
where $D_2>0$ and $D_{1+\beta}>0$ are some universal constants, and $B$ is
as defined in~\eqref{eq:B}.
\end{theorem}
\begin{proof}

Since $\alpha > 1$, by definition of a Fr\'echet distribution:
\begin{align}\label{eq:frechet}
\int_0^{\infty} \left(1-\exp(-TCx^{-\alpha})\right) dx =
(TC)^{1/\alpha}\Gamma\left(1\!-\!\tfrac{1}{\alpha}\right)
\end{align}
Notice that in~\eqref{eq:M} we have two terms of the form
$\exp(-TCx^{-\alpha})(TCx^{-\alpha})^p$, for $p=2$ and $p=\beta+1$.
In order to proceed, we first upper-bound the integral of such expression.
Through a change of variable (setting $t=TCx^{-\alpha}$) we get that for any
$p>0$:
\begin{align}
\label{eq:DP}
\int_0^{\infty} & \exp(-TCx^{-\alpha}) (TCx^{-\alpha})^p dx
= \frac{(TC)^{1/\alpha}}{\alpha}\int_0^{\infty} \exp(-t) t^{p-1-1/\alpha}
dt
 = D_p (TC)^{1/\alpha},
\end{align}
where $D_p = \Gamma(p-1/\alpha)/\alpha$ is bounded as long as $p> 1/\alpha$,
e.g.~if $p>1$.
From the definition of expectation 
we have that:
\begin{align*}
 \mathbb E  \left(\max_i X_i\right) = \int_0^{\infty} \mathbb P\left( \max_i
X_i \geq
x\right)dx
\end{align*}
We now bound the difference between this expectation 
and that of the Fr\'echet distribution.
\begin{align*}
&\left|\mathbb E \left(\max_i X_i\right) -
\int_0^{\infty} \left( 1- \exp\left(-TCx^{-\alpha}\right)\right) dx\right| \leq 
\\
&\qquad\leq \int_0^{\infty} 1 - \mathbb P\left( \max_i X_i \le x\right)dx +  
\int_0^{\infty}
\left(1 - \exp\left(-TCx^{-\alpha}\right)\right) dx\\
&\qquad\leq 
\left|\int_0^{B} 1 - \mathbb P\left( \max_i X_i \le x\right)dx 
+ \int_0^{B}\left(1 - \exp\left(-TCx^{-\alpha}\right)\right) dx \right|\\
&\qquad\quad + \left|\int_B^{\infty} 1 - \mathbb P\left( \max_i X_i \le 
x\right)dx 
 + \int_B^{\infty}\left(1 - \exp\left(-TCx^{-\alpha}\right)\right) 
dx \right|,
\end{align*}

where in the last term we split the domain of integration at $B$. We 
simply bound the first part by $B$ and for the second term, we use 
Lemma~\ref{lem:lardev} to obtain:
\begin{align*}
&\left|\int_B^{\infty} 1 - \mathbb P\left( \max_i X_i \le 
x\right)dx 
 + \int_B^{\infty}\left(1 - \exp\left(-TCx^{-\alpha}\right)\right) 
dx \right|\\
&\qquad\leq \int_0^{\infty} \left| \mathbb P\left( \max_i X_i 
\le 
x\right) -  \exp\left(-TCx^{-\alpha}\right)\right| 
dx  \\
&\qquad\leq \int_0^{\infty}
\exp\left(-TCx^{-\alpha}\right)\left( \frac{4}{T}\left(TCx^{-\alpha}\right)^2
+\frac{2C'}{C^{\beta+1}T^{\beta}}\left(TCx^{-\alpha}\right)^{\beta+1}  \right)
\end{align*}

Instantiating~\eqref{eq:DP}
for $p=2$ and $p=\beta+1$, we deduce that:
\begin{align*}
\Big| \mathbb E (\max_i X_i) -   (TC)^{1/\alpha}\Gamma(1-1/\alpha) \Big|
\leq B  + \frac{4D_2 }{T}(TC)^{1/\alpha}  + \frac{2C'D_{\beta+1}
}{C^{\beta+1}T^{\beta}}(TC)^{1/\alpha}
\end{align*}
Note that since $\alpha>1$, we know that $D_{\beta+1}$ and
$D_2$ are finite. This concludes the proof.
\end{proof}

\section{Proof of Lemma 1}

\setcounter{lemma}{1-1}
\begin{lemma}
 On $\xi$, we have that for any $k \leq K$, and for $T_{k,t} \geq N$,
 \begin{align}
\left( C_k n\right)^{\frac{1}{\alpha_k}} \Gamma\left(1\!-
\!\frac{1}{\alpha_k}\!\right)
\leq B_{k,t}
\!\leq\!(C_{k}n)^{\frac{1}{\alpha_{k}}}\Gamma\left(1\!-\!
\frac{1}{\alpha_k}\!\right)
\left(\!1\!+\!F\log(n)\sqrt{\log(n/\delta)}T_{k,t}^{-b/(2b+1)}
\!\right)
\end{align}
\end{lemma}

\begin{proof}
From Step~1, we know that on $\xi$, we can
bound $B_{k,t}$  as:
\begin{align*}
\Big( C_k n\Big)^{\frac{1}{\alpha_k}} \Gamma\left(1\!-\!
\frac{1}{\alpha_k}\right)
&\leq \Big( (\widehat C_{k,t}+B_2(T_{k,t}))  n\Big)^{\widehat h_{k,t} + 
B_1(T_{k,t})}
\bar \Gamma\big(\widehat h_{k,t} , B_1(T_{k,t})\big)\\
&\leq \Big( (C_{k}+2B_2(T_{k,t}))  n\Big)^{\frac{1}{\alpha_{k}} + 2B_1(T_{k,t})}
\bar \Gamma\big(1/\alpha_k, 2B_1(T_{k,t})\big),
\end{align*}
since $\Gamma$ is decreasing on $[0,1]$.


Note that by Theorem~\ref{cor:mean} we know that
 $(C_k n)^{1/\alpha_k} \Gamma(1 - 1/\alpha_k)$ is a proxy for the expected
maximum of the arm distribution with tail index
$\alpha_k$. Factoring out $(C_k n)^{1/\alpha_k}$ we get:
\begin{align}
\left( \left(C_{k}+2B_2\left(T_{k,t}\right)\right)
n\right)^{1/\alpha_{k} + 2B_1(T_{k,t})}
\bar \Gamma\left(1/\alpha_k, 
2B_1(T_{k,t})\right)\qquad\qquad\qquad\qquad\qquad\qquad\nonumber\\
\qquad\leq (C_{k}n)^{2B_1(T_{k,t})}\bar \Gamma\left(1/\alpha_k,
2B_1\left(T_{k,t}\right)\right) (C_{k}n)^{\frac{1}{\alpha_{k}}}
\left(1+\frac{2B_2\left(T_{k,t}\right)}{C_k}\right)^{1/\alpha_{k} +
2B_1(T_{k,t})}
\label{eq:expr1}
\end{align}
As we pull each arm at least $N$ times (by the assumptions) we have that
$T_{k,t} \geq N$ which implies $\max(2B_1(N), 2B_2(N)/C_k) \leq 1$.
Since $\alpha_k >1$:
\begin{align}
 \bigg(1+  \frac{2B_2 ( T_{k,t} ) }{C_k}\bigg)^{1/\alpha_{k} +
2B_1(T_{k,t})}
&\leq \Big(1+\frac{2E}{C_k} \sqrt{
\log(T_{k,t}/\delta)} \log(T_{k,t}) T_{k,t}^{-b/(2b+1)}\Big)^2\nonumber\\
&\leq 1+\frac{6E}{C_k} \sqrt{
\log(n/\delta)} \log (n) T_{k,t}^{-b/(2b+1)}.\label{eq:expr2}
\end{align}
Using again $T_{k,t} \geq N$ and  $\log(C_kn) D
\sqrt{\log(1/\delta)}
N^{-b/(2b+1)} \leq 1/2$ for all $k$, we have:
\begin{align}
(C_kn)^{2B_1(T_{k,t})} &= \exp(\log(C_kn) 2 D \sqrt{\log(1/\delta)}
T^{-b/(2b+1)}) \nonumber \\
&\leq 1\!+\!2\log(C_kn) D \sqrt{\log(1/\delta)} T^{-b/(2b+1)}.\label{eq:expr3}
\end{align}
Now, let $c$ be the maximum of the absolute value of the derivative of 
$\Gamma$ on the 
segment: $$\left[1 - \max_k \frac{1}{\alpha_k} - D \sqrt{\log(1/\delta)}
N^{-b/(2b+1)}, 1 -  \min_k \frac{1}{\alpha_k} + D \sqrt{\log(1/\delta)}
N^{-b/(2b+1)}\right]$$ 

Since by the assumption on $N$:
$$ N > \left(\frac{2D \sqrt{\log(1/\delta)}}{1-\max_k 1/\alpha_k}
\right)^{(2b+1)/b},$$
we know that $c$ is smaller than the maximum of the absolute value of the
derivative of $\Gamma$ function in 
$\left[\frac12(1 - \max_k 1/\alpha_k),\frac32(1 -\min_k 1/\alpha_k)\right],$ 
since $\Gamma$ is a convex
and decreasing function on $[0,1]$. When $\xi$ happens, this implies:
\begin{align}
\nonumber\bar \Gamma (1/\alpha_k, 2B_1(T_{k,t})\big)
& \leq  \Gamma(1\! -\! 1/\alpha_k) + 2 c B_1(T_{k,t}) \\
& \leq  \Gamma(1\! -\! 1/\alpha_k) + 2 c D \sqrt{\log(1/\delta)} 
T_{k,t}^{-b/(2b+1)}\label{eq:expr4}
\end{align}

Finally, combining~\eqref{eq:expr1} and~\eqref{eq:expr4}, we get:
\begin{align*}
&\Big( (C_{k}+2B_2(T_{k,t}))  n\Big)^{1/\alpha_{k} + 2B_1(T_{k,t})} \bar
\Gamma\big(1/\alpha_k, 2B_1(T_{k,t})\big)\\
&\qquad\leq (C_{k}n)^{1/\alpha_{k}}\Gamma(1- 1/\alpha_k)
\big(1+F\log(n)\sqrt{\log(n/\delta)}T_{k,t}^{-b/(2b+1)}\big),
\end{align*}
where $F$ depends on $(\alpha_k, C_k)_k, C',D,$ and $E$.

This implies that for $T_{k,t} \geq N$, we can bound $B_{k,t}$ as:
\begin{align*}
\left( C_k n\right)^{1/\alpha_k} \Gamma\left(1\!-\!
\alpha_k^{-1}\right)
\leq B_{k,t}  
\leq\!(C_{k}n)^{1/\alpha_{k}}\Gamma\left(1\!-\!
\alpha_k^{-1}\right)
\left(\!1\!+\!F\log(n)\sqrt{\log(n/\delta)}T_{k,t}^{-b/(2b+1)}
\!\right) 
\end{align*}

\end{proof}

\section{Proof of Lemma 2}

\begin{lemma}
Let $X_1, \ldots, X_T$ be i.i.d.~samples from an
($\alpha, \beta, C, C'$)-second
order Pareto distribution $F$. Let $\xi'$ be an event of
probability larger than $1-\delta$. Then for $T$ large enough so that
$c\max(1/T, 1/T^{\beta}) \leq 1/2$ (for a given constant
$c>0$ that depends only on $C,C'$ and $\beta$), for $T\geq
\log(2)\max\left(C\left(2C'\right)^{1/\beta}, 8\log\left(2\right)\right)$, and
for $\delta < 1/2$:
\begin{align*}
\!\mathbb E \left[\max_{t\leq T} X_{t}\mathbf 1\{\xi'\}\right]
&\geq(TC)^{1/\alpha} \Gamma(1\!-\!1/\alpha)
- 8(TC)^{1/\alpha} \delta^{1-1/\alpha}\\
&\quad - 2\left(\tfrac{4D_2 (TC)^{1/\alpha}}{T}  +
\tfrac{2C'D_{1+\beta} (TC)^{1/\alpha}}{C^{1+\beta}T^{\beta}} +  B\right) .
\end{align*}
\end{lemma}
\begin{proof}
Since the probability of $\xi'$ is larger than $1-\delta$:
\begin{align*}
\mathbb E \left[\max_{t\leq T} X_{t}\mathbf 1\{\xi'\} \right] &= \mathbb E 
\left[\max_{t\leq T}
X_{t}\right] -  \mathbb E\left[ \left(\max_{t\leq T} X_{t}\right) \mathbf 
1\{\xi'^C\}\right]\\
&= \mathbb E \left[\max_{t\leq T} X_{t}\right] - \int_{0}^{\infty} \mathbb
P\left[\left(\max_{t\leq T} X_{t}\right) \mathbf 1\{\xi'^C\} > x\right] dx\\
&\geq \mathbb E \left[\max_{t\leq T} X_{t}\right] - 
\int_{x_{\delta}}^{\infty} \mathbb
P\left[\left(\max_{t\leq T} X_{t}\right) > x\right] dx - \delta x_{\delta},
\end{align*}
where $x_{\delta}$ is such that $P\left(\max_{t\leq T} X_{t}\leq 
x_{\delta}\right)
= 1-\delta$.

Since we have $T\geq
\log(2)\max\left(C\left(2C'\right)^{1/\beta}, 8\log\left(2\right)\right)$, and
$\delta < 1/2$, we get by Lemma~\ref{lem:lardev}:
\begin{align*}
&\left|\mathbb P\left( \max_i X_i \leq
\left(TC/\log\left(1/(1-\delta)\right)\right)^{1/\alpha}\right) -
(1-\delta)\right|\\
&\qquad\qquad \leq  
(1-\delta)\left(\tfrac{4}{T}\left(\log\tfrac{1}{1-\delta}\right)^2
+\tfrac{2C'}{C^{1+\beta}}\left(\log\tfrac{1}{1-\delta}\right)^{1+\beta}\right)\\
&\qquad\qquad \leq \tfrac{4}{T}(2\delta)^2
+\tfrac{2C'}{C^{1+\beta}}(2\delta)^{1+\beta}\\
&\qquad\qquad \leq c\delta\max\left(\tfrac{\delta}{T},
\tfrac{\delta^{\beta}}{T^{\beta}}\right)\\
&\qquad\qquad \leq c \delta\max\left(\tfrac{1}{T}, \tfrac{1}{T^{\beta}}\right),
\end{align*}
where $c>0$ is a constant that depends only on $C,C'$ and $\beta$.
This implies that for $T$ large enough so that $c\max(1/T,
1/T^{\beta}) \leq 1/4$: 
\begin{align*}
\bar x = \left(TC/\log\left(1/(1-\delta/2)\right)\right)^{1/\alpha}
&\geq x_{\delta} \geq
\left(TC/\log\left(1/(1-2\delta)\right)\right)^{1/\alpha}=\underline x.
\end{align*}

By Theorem~\ref{cor:mean} we can now deduce that:
\begin{align*}
 \mathbb E \left[\max_{t\leq T} X_{t}\mathbf 1\{\xi'\} \right]
&\geq \mathbb E \left[\max_{t\leq T} X_{t}\right] - \int_{\underline 
x}^{\infty} \mathbb
P\left[\left(\max_{t\leq T} X_{t}\right) > x\right] dx - \delta \bar x\\
&\geq \mathbb E \left[\max_{t\leq T} X_{t}\right] - \int_{\underline 
x}^{\infty} \left(1 -
\exp\left(-TCx^{-\alpha}\right)\right) dx\\
&\quad - \left(\tfrac{4D_2 }{T}\left(TC\right)^{1/\alpha}  +
\tfrac{2C'D_{1+\beta} }{C^{1+\beta}T^{\beta}}\left(TC\right)^{1/\alpha} +  
B\right)  - \delta
\bar x.
\end{align*}
By the method of substitution, the Taylor expansion and for $\delta$ small 
enough:
\begin{align*}
\int_{\underline x}^{\infty} \left(1 - \exp\left(-TCx^{-\alpha}\right)\right) dt
&= \frac{\left(TC\right)^{1/\alpha}}{\alpha} 
\int_0^{\log\left(1/(1-2\delta)\right)} \big(1
- \exp(-t)\big)t^{-1-1/\alpha} dt\\
&\leq  
\frac{2\left(TC\right)^{1/\alpha}}{\alpha}\int_0^{\log\left(1/(1-2\delta)\right)
} 
\exp(-t) 
t^{-1/\alpha}
dt\\
&\leq \frac{2\left(TC\right)^{1/\alpha}}{\alpha} 
\int_0^{\log\left(1/(1-2\delta)\right)}
t^{-1/\alpha}
dt\\
&\leq \frac{2\left(TC\right)^{1/\alpha}}{\alpha} 
\log\left(1/(1-2\delta)\right)^{1-1/\alpha} 
\\
&\leq
\frac{8}{\alpha-1}\left(TC\right)^{1/\alpha} \delta^{1-1/\alpha}.
\end{align*}


We conclude the proof by combining all the above with the fact that  $\delta
\bar x \leq 4\left(TC\right)^{1/\alpha} \delta^{1-1/\alpha}$ and with
Theorem~\ref{cor:mean} to get the final lower-bound on $\mathbb E 
\left[\max_{t\leq T}
X_{t}\mathbf 1\{\xi'\}\right]$:
\begin{align*}
\left(TC\right)^{1/\alpha} \Gamma(1\!-\!1/\alpha) - 
\left(4\!+\!\tfrac{8}{\alpha-1}\right)\left(TC\right)^{1/\alpha}
\delta^{1-1/\alpha} \!-\! 2\left(\tfrac{4D_2 }{T}\left(TC\right)^{1/\alpha} \! 
+\!
\tfrac{2C'D_{1+\beta} }{C^{1+\beta}T^{\beta}}\left(TC\right)^{1/\alpha}\! +\!  
B\right) .
\end{align*}
\end{proof}

\end{document}